\documentclass[12pt]{article}
\usepackage{amsmath}
\usepackage{graphicx,psfrag,epsf}
\usepackage{enumerate}
\usepackage{url} 
\usepackage{float}

\usepackage{algorithm,algorithmic}
\usepackage[algo2e]{algorithm2e}
\usepackage{multirow}
\usepackage{lipsum}
\usepackage{amsfonts}
\usepackage{cite}      
                  
\usepackage{graphicx}   

\usepackage{amsmath}
\usepackage{amsthm}
\usepackage{amsfonts}
\usepackage{amssymb}
\usepackage{graphicx}
\usepackage{subfigure}
\usepackage{transparent}
\usepackage{multirow}
\usepackage{color}
\usepackage{blkarray}
\usepackage[round]{natbib}
\usepackage{booktabs}
\usepackage[dvipsnames]{xcolor}
\usepackage[]{hyperref}
\hypersetup{colorlinks=true, 
	linkcolor=red,       
    citecolor=blue,        
    filecolor=magenta,      
    urlcolor=cyan           
}  

\newcommand{\bm}[1]{\boldsymbol{#1}}
\newcommand{\E}{\mathbb{E}}

\renewcommand{\d}{\mathrm{d}}
\renewcommand{\P}{\mathbb{P}}

\newcommand{\Y}{\bm{Y}}

\newcommand{\mE}{\mathcal{E}}

\newcommand{\mB}{\mathcal{B}}

\newcommand{\mT}{\mathcal{T}}
\newcommand{\wt}[1]{\widetilde{#1}}

\newcommand{\F}{\mathcal{F}}

\newcommand{\iid}{\stackrel{iid}{\sim}}

\newcommand{\wh}[1]{\smash{\widehat{#1}}}

\usepackage{array}
\newcolumntype{C}{@{\extracolsep{0.5in}}c@{\extracolsep{0pt}}}

\def\C {\,|\:}

\def\C {\,|\:}

\def\mF{\mathcal{F}}

\def\B{\bm{B}}
\def\b{\bm{\beta}}
\def\Y{\bm{Y}}

\def\x{\bm{x}}

\def\mV{\mathcal{V}}

\def\b{\bm{\beta}}

\renewcommand{\d}{\mathrm{d}\,}
\newcommand{\e}{\mathrm{e}}

\newcommand{\N}{\mathbb{N}}
\newcommand{\R}{\mathbb{R}}
\newcommand{\Ha}{\mathcal{H}^\alpha}

\newtheorem{definition}{Definition}[section]
\newtheorem{lemma}{Lemma}[section]

\newtheorem{theorem}{Theorem}[section]

\newtheorem{remark}{Remark}[section]
\newtheorem{corollary}{Corollary}[section]

 \theoremstyle{assumption}

\begin{document}

\def\spacingset#1{\renewcommand{\baselinestretch}%
{#1}\small\normalsize} \spacingset{1}


  \title{\sf 
   On Theory for BART}

 \spacingset{1.3}
  
  \author{
    Veronika Ro\v{c}kov\'{a}\footnote{ Assistant Professor in Econometrics and Statistics at the {University of Chicago Booth School of Business}. This research was supported  by  the James S. Kemper Foundation Faculty Research Fund.}
     \and Enakshi Saha\footnote{
    3$^{rd}$ year PhD Student at the  {Department of  Statistics, University of Chicago}.
    } \hspace{0.3cm}
} 
  
   \maketitle

\bigskip

\begin{abstract}
Ensemble learning is a statistical paradigm 
built on the premise that  many weak learners can perform exceptionally well when deployed collectively.
The BART method of \cite{chipman2010bart} is a prominent example of {\sl Bayesian} ensemble learning, where each learner is a tree.
Due to its impressive performance, BART has received a lot of attention from practitioners.
Despite its wide popularity, however, theoretical studies of BART have  begun emerging only very recently.
Laying the foundations for the theoretical analysis of Bayesian forests,  \cite{rockova2017posterior} showed optimal posterior concentration   under {\sl conditionally uniform tree priors.}
These priors  deviate from the actual priors implemented in BART. Here, we study the exact BART prior and propose a simple modification so that  it {\sl also} enjoys optimality properties.
To this end, we dive into branching process theory. We obtain  tail bounds for the distribution of total progeny under heterogeneous Galton-Watson (GW) processes exploiting their connection to random walks.
We conclude with a result stating the optimal rate of posterior convergence for BART.
\end{abstract}

\clearpage
 \spacingset{1.5}
 
\section{Bayesian Machine Learning}


Bayesian Machine Learning  and Bayesian Non-parametrics share the same objective: increasing flexibility necessary to address very complex problems
 using a Bayesian approach with minimal subjective input. While the two fields can be, to some extent,  regarded as  synonymous, their emphasis is quite different.
Bayesian non-parametrics has evolved into a largely theoretical field, studying frequentist properties of posterior objects in inifinite-dimensional parameter spaces.
Bayesian machine learning, on the other hand, has been primarily concerned with developing scalable tools for computing such posterior objects.  In this work, we bridge these two fields by providing theoretical insights into one of the workhorses of Bayesian machine learning, the BART method.

Bayesian Additive Regression Trees (BART) are one of the more widely used   Bayesian prediction tools and their popularity continues to grow.
Compared to its competitors (e.g. Gaussian processes, random forests or neural networks) BART requires considerably less  tuning, while  maintaining  robust and relatively scalable performance  (\texttt{BART}  \rm{R} package of McCulloch (2017), \texttt{bartMachine} \rm{R} package of \cite{bleich2014variable}, top down particle filtering of \cite{lakshminarayanan2013top}).  BART has been successfully deployed in many prediction tasks, often outperforming its competitors (see predictive comparisons on $42$ data sets in \cite{chipman2010bart}). More recently, its flexibility and stellar prediction  has been capitalized on
in causal inference  tasks for heterogeneous/average treatment effect estimation (\cite{hill2011bayesian}, \cite{hahn2017bayesian} and references therein).  BART has also served as a springboard for various incarnations and extensions including: Monotone BART (\cite{chipman2016high}), Heteroscedastic BART (\cite{pratola2017heteroscedastic}), treed Gaussian processes (\cite{gramacy2008bayesian}) and dynamic trees (\cite{taddy2011dynamic}), to list a few. Related non-parametric constructions based on recursive partitioning have proliferated in the Bayesian machine learning community for  modeling relational data (Mondrian process of \cite{roy2008mondrian}, Mondian forests (\cite{lakshminarayanan2014mondrian}). In short, BART continues to have a decided impact on the field of Bayesian non-parametrics/machine learning.

Despite its widespread  popularity, however, the theory has not caught up with its applications. First theoretical results were obtained only very recently. As a precursor to these developments, \cite{coram2006consistency} obtained a consistency result for Bayesian histograms in binary regression with a single predictor. \cite{rockova2017posterior2} provided a posterior concentration result for Bayesian regression histograms in Gaussian non-parametric regression, also with  one predictor.
\cite{rockova2017posterior} (further referred to as RP17) then extended their study to trees and forests in a high-dimensional setup  where $p>n$ and where variable selection uncertainty is present. They obtained the first theoretical results for Bayesian CART, showing optimal posterior 
 concentration (up to a log factor) around a $\nu$-H\"{o}lder continuous regression function (with a smoothness $0<\nu\leq1$). Going further, they also show  optimal performance for Bayesian forests, both in additive and non-additive regression.
  \cite{linero2017bayesian} obtained similar results for Bayesian ensembles, but  for {\sl fractional} posteriors (raised to a power). 
     The proof of RP17, on the other hand, relies on a careful construction of sieves and applies to regular posteriors. 
 In addition, \cite{linero2017bayesian}  do not study step functions (the essence of Bayesian CART and BART) but aggregated smooth kernels, allowing  for $\nu>1$. 
  Building on RP17, \cite{liu2018abc} obtained model selection consistency results (for variable and regularity selection) for Bayesian forests. 

  

Albeit related, the tree priors studied in RP17 are {\sl not} the actual priors deployed in BART.
Here, we develop new tools  for the analysis of the actual BART prior and obtain parallel results to those in RP17.
To begin, we dive into branching process theory to characterize aspects of the distribution on  total progeny under heterogeneous Galton-Watson processes.
Revisiting several useful facts about Galton-Watson processes, including their connection to random walks, we derive a new prior tail bound for the tree size under the BART prior.
With our proving strategy,  the actual prior of \cite{chipman2010bart} {\sl does not} appear to penalize large trees aggressively enough.
We suggest a very simple modification of the prior by altering the splitting probability. With this minor change,  the prior is shown to induce the right amount of regularization and 
optimal speed of posterior convergence.

The paper is structured as follows. Section 2 revisits trees and forests in the context of non-parametric regression and discusses the BART prior. Section 3 reviews the notion of posterior concentration. Section 4 discusses Galton Watson processes and their connection to Bayesian CART. Section 5 is concerned with tail bounds on total progeny. Section 6 and 7 describe prior and concentration properties of BART. Section 7 wraps up with a discussion.

\vspace{-0.5cm}
\section{The Appeal of Trees/Forests}

The data setup under consideration consists of  $Y_i\in\R$, a set of low dimensional outputs, and $ \x_i = (x_{i1} , \ldots , x_{ip} )' \in [0,1]^p $,  a set of high dimensional inputs for $1\leq i\leq n$.  Our statistical framework is    non-parametric regression, which characterizes the input-output relationship  through
$$
Y_i=f_0(\x_i)+\varepsilon_i,\quad \varepsilon_i\iid\mathcal{N}(0,1),
$$
where $f_0:[0,1]^p\rightarrow \R$ is an unknown regression function. A regression tree can be used to reconstruct $f_0$  via a mapping $f_{\mT,{\b}}:[0,1]^p\rightarrow\R$   so that $f_{\mT,{\b}} ( \x ) \approx f_0 ( \x ) $ for $ \x \notin \{ \x_i \}_{i=1}^n $.
Each such mapping   is essentially a step function
\begin{equation}\label{eq:tree_mapping}
f_{\mT,{\b}}(\x)=\sum_{k=1}^K\beta_k\mathbb{I}(\x\in\Omega_k)
\end{equation}
underpinned by a tree-shaped partition $\mT=\{\Omega_k\}_{k=1}^K$ and  a vector of step heights $\b=(\beta_1,\dots,\beta_K)'$.
The vector $\b$ represents quantitative guesses of the average outcome inside each cell.
Each  partition $\mT$ consists of rectangles obtained by recursively applying a splitting rule (an axis-parallel bisection of the predictor space). 
We focus on {\sl binary} tree partitions, where each internal node (box) is split into two children  (formal definition below). 

\begin{definition} (A Binary Tree Partition)\label{def:tree_part}
A binary tree partition $\mT=\{\Omega_k\}_{k=1}^K$ consists of  $K$ rectangular cells $\Omega_k$  obtained with $K-1$ successive recursive binary splits  of the form $\{\x_j \le c \}$ vs $\{\x_j > c\}$ for some $j\in\{1,\dots,p\}$, where the splitting value $c$ is chosen from {\sl observed values} $\{x_{ij}\}_{i=1}^n$. 
\end{definition}

Partitioning is intended to increase within-node homogeneity of outcomes. 
In the traditional CART method (\cite{brieman1984classification}), the tree is obtained by ``greedy growing" (i.e. sequential optimization of some impurity criterion) until homogeneity cannot be substantially improved. The tree growing process  is often followed by ``optimal pruning" to increase generalizability. Prediction  is then determined by terminal nodes of the pruned tree and takes the form either of a class level in classification problems, or the average of the response variable in least squares regression problems (\cite{brieman1984classification}). 

In tree {\sl ensemble} learning, each  constituent  is designed to be a weak learner, addressing a slightly different aspect of the prediction problem. These trees are intended to be shallow and  are woven  into a forest mapping
\begin{equation}\label{eq:forest_mapping}
f_{\mE,{\B}}(\x)=\sum_{t=1}^Tf_{\mT_t,{\b}_t}(\x),
\end{equation}
where each $f_{\mT_t,{\b}_t}(\x)$ is of the form \eqref{eq:tree_mapping}, $\mE=\{\mT_1,\dots,\mT_T\}$ is an ensemble of trees and $\B=\{\b_1,\dots,\b_T\}'$ is a collection of jump sizes for the $T$ trees. Random forests obtain each tree learner from a bootstrapped version of the data. Here, we consider a Bayesian variant, the BART method of \cite{chipman2010bart}, which relies on the posterior distribution over $f_{\mE,\B}$ to reconstruct the unknown regression function $f_0$.


\subsection{Bayesian Trees and Forests}
Bayesian CART was introduced as a Bayesian alternative to CART, where regularization/stabilization is obtained with a prior rather than with pruning (\cite{chipman1998bayesian}, \cite{denison1998bayesian}). The prior distribution is assigned  over a class of step functions
$$
\mF=\{ f_{\mE,\B}(\x) \quad\text{of the form \eqref{eq:forest_mapping} for some $\mE$ and $\B$} \}
$$ 
in a hierarchical manner. 

 
The BART prior by \cite{chipman2010bart} assumes that the number of trees $T$ is fixed. 
The authors recommend a default choice $T=200$ which was seen to provide good results. 
Next, the tree components $(\mT_t, \b_t)$ are a-priori independent of each other in the sense that
\begin{equation}\label{eq:joint_prior}
\pi(\mE,\B)=\prod_{t=1}^T\pi(\mT_t)\pi(\b_t\C\mT_t),
\end{equation}
where $\pi(\mT_t)$ is the prior probability of a partition $\mT_t$ and $\pi(\b_t\C\mT_t)$ is the prior distribution over the jump sizes.

\subsubsection{\bf  Prior on Partitions $\pi(\mT)$}\label{subsub:tree} 

In BART and Bayesian CART of \cite{chipman1998bayesian}, the prior over trees is  specified implicitly as a tree generating stochastic process, described as follows:
\begin{enumerate}
\item Start with a single leave (a root node) $[0,1]^p$.
\item Split a terminal node, say $\Omega_t$, with a probability
\begin{equation}\label{eq:p_split}
p_{split}(\Omega_t)=\frac{\alpha}{(1+d(\Omega_t))^{\gamma}}
\end{equation}
for some $\alpha\in(0,1)$ and $\gamma\geq 0$,
where $d(\Omega_t)$ is the depth of the node $\Omega_t$ in the tree architecture.
\item If the node $\Omega_t$ splits, assign a splitting rule  and create left and right children nodes. The splitting rule consists of picking a split variable $j$ uniformly from available directions $\{1,\dots, p\}$ and
picking a split point $c$ uniformly from available data values $x_{1j},\dots, x_{nj}$. 
Non-uniform priors can also be used to favor splitting values that are thought to be more important. For example,  splitting values can be given more weight towards the center and less weight towards the edges.
\end{enumerate}

\subsubsection{\bf Prior on Step Heights $\pi(\b\C\mT)$}\label{subsub:jump} Given a tree partition $\mT_t$ with $K_t$ steps, we consider  iid Gaussian jumps
$$
\pi(\b_t\C\mT_t)=\prod_{k=1}^{K_t}\phi(\beta_{tj};0,1/T),
$$
where $\phi(x;0,\sigma^2)$ is a Gaussian density with mean $0$ and variance $\sigma^2$. \cite{chipman2010bart} recommend first shifting and rescaling $Y_i$'s so that the observed transformed values range from -0.5 to 0.5. Then they assign a conjugate normal prior $\beta_{tj} \sim N(0, \sigma^2)$, where $\sigma = 0.5/k\sqrt{T}$ for some suitable value of $k$. This is to ensure that the prior assigns substantial probability to the range of the $Y_i$'s. \\

The BART prior also involves  an inverse chi-squared distribution on residual variance,  with hyper-parameters  chosen so that the $q^{th}$ quantile of the prior is located at some sample based variance estimate. While the case of random variance can be incorporated in our analysis (\cite{de2010adaptive}), we will for simplicity assume that the residual variance is fixed.

Existing theoretical work for Bayesian forests (RP17) is available for a different prior on tree partitions $\mT$.  Their analysis assumes a hierarchical prior consisting of (a) a prior on the size of a tree $K$ and (b) a uniform prior over trees of size $K$. This prior is equalitarian in the sense that
 trees with the same number of leaves are a-priori equally likely regardless of their topology. {RP17 also imposed a diversification restriction in their prior, focusing on $\delta$-valid ensembles (Definition 5.3) which consist of trees that do not overlap too much. }
   The prior on the number of leaves $K$ is a very important ingredient for regularization. We will study aspects of its distribution under the actual BART prior in later sections.

\vspace{-0.5cm}
\section{Bayesian Non-parametrics Lense}
One way of assessing the quality of a Bayesian procedure is by studying the learning rate of its posterior, i.e. the speed  at which the posterior distribution shrinks around the truth as $n\rightarrow\infty$. These statements are ultimately framed in a frequentist way, describing the typical behavior of the posterior under the true generative model $\P_{f_0}^{(n)}$. 
Posterior concentration rate results have been  valuable for the proposal and calibration of priors. In infinite-dimensional parameter spaces, such as the one considered here, seemingly innocuous priors can lead to inconsistencies (\cite{cox1993analysis}, \cite{diaconis1986consistency}) and far more care has to be exercised to come up with well-behaved priors. 

The Bayesian approach requires placing a prior measure $\Pi(\cdot)$ on   $\mF$,  the set of qualitative guesses of  $f_0$. Given observed data $\Y^{(n)}=(Y_1,\dots, Y_n)'$,  inference about $f_0$ is then carried out via the posterior distribution
$$
\Pi(A\mid\Y^{(n)})=\frac{\int_A \prod_{i=1}^n \Pi_f(Y_i\mid\x_i)\d\Pi(f)}{\int \prod_{i=1}^n \Pi_f(Y_i\mid\x_i)\d\Pi(f)}\quad\forall A\in\mathcal{B}
$$
where $\mathcal{B}$ is a $\sigma$-field on $\mF$ and 
where $\Pi_f(Y_i\C\x_i)$ is the likelihood function for the output $Y_i$ under $f$.

In Bayesian non-parametrics, one of the usual goals is determining
\emph{how fast the posterior probability measure concentrates around $f_0$} as $n\rightarrow\infty$.
This speed can be assessed by inspecting the size of the smallest  $\|\cdot\|_n$-neighborhoods around $f_0$ that contain most of the posterior probability (\cite{ghosal2007convergence}), where  $\|f\|_n^2=\frac{1}{n}\sum_{i=1}^nf(\x_i)^2$.
For a diameter $\varepsilon>0$ and some $M>0$, we denote with 
{$$
A_{\varepsilon, M}=\{f_{\mE,\B} \in\mF :\|f_{\mE,\B} -f_0\|_n\leq M\,\varepsilon \}
$$} 
the $M\varepsilon$-neighborhood centered around $f_0$. We say that the posterior distribution concentrates at speed $\varepsilon_n\rightarrow 0$  such that $n\,\varepsilon_n^2\rightarrow\infty$ when
\begin{equation}\label{concentration}
\Pi(A_{\varepsilon_n, M_n}^c \C \Y^{(n)} ) \rightarrow 0\quad\text{in $\P_{f_0}^{(n)}$-probability as $n\rightarrow\infty$} 
\end{equation}
 for any $M_n\rightarrow\infty$. Posterior consistency statements are a bit weaker, where $\varepsilon_n$ in \eqref{concentration}  is replaced with a fixed neighborhood $\varepsilon>0$.
 We will position our results using  {$\varepsilon_n=n^{-\nu/(2\nu+p)}\log^{1/2}n$}, the near-minimax rate for estimating a $p$-dimensional $\nu$-smooth function.  We will also assume that $f_0$ is 
 H\"{o}lder continuous, i.e. $\nu$-H\"{o}lder smooth with $0<\nu\leq 1$. The limitation $\nu\leq 1$ is an unavoidable consequence of  using step functions to approximate smooth $f_0$ and can be avoided with smooth kernel methods (\cite{linero2017bayesian}).
  
The statement \eqref{concentration} can be proved by verifying the following three conditions (suitably adapted from Theorem 4 of \cite{ghosal2007convergence}):

\begin{equation}\label{eq:entropy1}
\displaystyle \sup_{\varepsilon > \varepsilon_n} \log  N\left(\tfrac{\varepsilon}{36}; A_{{\varepsilon},1}\cap\mF_n; \|.\|_n\right) \leq n\,\varepsilon_n^2
\end{equation}
\begin{equation}\label{eq:prior1}
\displaystyle {\Pi(A_{\varepsilon_n,1})}\geq \e^{-d\,n\,\varepsilon_n^2}
\end{equation} 
\begin{equation}\label{eq:remain1}
\displaystyle \Pi(\mathcal{F} \backslash \mathcal{F}_n) = o(\e^{-(d+2)\,n\,\varepsilon_n^2})
\end{equation}
 for some $d>2$. In \eqref{eq:entropy1},  $N(\varepsilon; \Omega; d)$ is the $\varepsilon$-covering number of a set $\Omega$ for a semimetric $d$,  i.e. the minimal number of $d$-balls of radius $\varepsilon$ needed to cover a set $\Omega$. 
A few remarks are in place. The condition \eqref{eq:remain1} ensures that the prior zooms in on smaller, and thus more manageable, sets of models $\mathcal{F}_n$ by assigning only a small probability outside these sets. The condition \eqref{eq:entropy1} is known as ``the  entropy condition" and controls the  combinatorial richness of the approximating sets $\mathcal{F}_n$. 
Finally, condition \eqref{eq:prior1} requires that the prior charges an $\varepsilon_n$ neighborhood of the true function. 
The results of type \eqref{concentration} quantify not only the typical distance between a point estimator (posterior mean/median) and the truth, but also the typical spread of the posterior around the truth. 
These results are typically the first step towards further uncertainty quantification statements.

\vspace{-0.5cm}
\section{The Galton-Watson Process Prior}
The Galton-Watson (GW) process provides a mathematical representation of an evolving  population of individuals who  reproduce and die subject to laws of chance.
Binary tree partitions $\mT$ under the prior \eqref{eq:p_split} can be thought of as  realizations of such a branching process.
Below, we review some terminology of branching processes and link them to Bayesian CART.

We denote with $Z_t$ the population size at time $t$ (i.e. the number of nodes in the $t^{th}$ layer of the tree).
The process starts at time $t=0$ with a single individual, i.e. $Z_0=1$. At  time $t$, each member is split {\sl independently} of one another into a random number of offsprings. 
Let $Y_{ti}$ denote the number of offsprings produced by the $i^{th}$ individual of the $t^{th}$ generation and let $g_t(s)$ be the associated probability generating function.
A binary tree is obtained when
each node has either {\sl zero} or {\sl two} offsprings, as characterized by
\begin{equation}\label{pgf}
g_t(s)=s^0\P(Y_{t1}=0)+s^2\P(Y_{t1}=2),\quad 0\leq s\leq 1.
\end{equation}
Homogeneous GW process is obtained when all $Y_{ti}$'s are iid.
A {\sl heterogeneous} GW process is a generalization where the offspring distribution is allowed to vary according to the generations,
i.e. the variables $Y_{ti}$ are independent but {\sl non-identical}. The Bayesian CART prior of \cite{chipman1998bayesian}  can be framed as a heterogeneous GW process, where the probability of splitting a node (generating offsprings) depends on the depth $t$ of the node in the tree. 
In particular, using \eqref{eq:p_split} one obtains for $0<\alpha<1$ and $\gamma>0$
\begin{equation}\label{split_prob}
\P(Y_{t1}=2)=1-\P(Y_{t1}=0)=\frac{\alpha}{(1+t)^{\gamma}}.
\end{equation}
The population size at time $t$ satisfies $Z_{t}=\sum_{i=1}^{Z_{t-1}} Y_{ti}$ and its expectation can be written as
$$
\E Z_t=\E[\E (Z_t  \mid Z_{t-1})]=(2\alpha)^t[(t+1)!]^{-\gamma}.
$$
Since $\E Z_1<1$ under \eqref{split_prob}, the process is subcritical and thereby it dies out with probability one. This means that the random sequence $\{Z_t\}$ consists of zeros for all but a finite number of  $t$'s. The overall number of nodes in the tree (all ancestors in the family pedigree)  
\begin{equation}\label{eq:pop_size}
X=\sum_{t=0}^\infty Z_t
\end{equation}
is thus finite with probability one. The number of {\sl leaves} (bottom nodes) $K$ can be related to $X$ through
\begin{equation}\label{eq:progeny}
K=(X+1)/2
\end{equation}
and satisfies
\begin{equation}\label{eq:Tex_X}
T_{ex}+1\leq K\leq 2^{T_{ex}},
\end{equation}
 where $T_{ex}=\mathrm{min}\{t:Z_t=0\}$ is the time of extinction. 
In \eqref{eq:Tex_X}, we have used the fact that  $T_{ex}-1$ is the depth of the tree, where
the  lower bound is obtained with asymmetric trees with only one node split at each level and the upper bound is obtained with  symmetric full binary trees (all nodes are split at each level).

Regularization is an essential remedy against overfitting and Bayesian procedures have a natural way of doing so
through a prior. In the context of trees, the key regularization element is the prior on the number of bottom leaves $K$,  
which is completely characterized by  the distribution of total progeny $X$ via \eqref{eq:progeny}.
Using this connection, in the next section we study the tail bounds of the distribution $\pi(K)$ implied by the Galton-Watson process.

\section{Bayesian Tree Regularization}
If we knew $\nu$, the optimal (rate-minimax) choice of the number of tree leaves would be $K\asymp K_\nu=n^{p/(2\nu+p)}$ (RP17). When $\nu$ is unknown,  one can do almost as well (sacrificing only a log factor in the convergence rate) using a suitable prior $\pi(K)$. As noted by \cite{coram2006consistency}, the tail behavior of $\pi(K)$ is critical for controlling  the vulnerability/resilience to overfitting. The anticipation is  that with smooth $f_0$, more rapid posterior concentration takes place when $\pi(K)$ has a heavier tail. However, too heavy tails make it easier to overfit when the true function is less smooth. To achieve an equilibrium, \cite{denison1998bayesian} suggest the Poisson distribution (constrained to $\mathbb{N}\backslash\{0\}$), which satisfies 
\begin{equation}\label{eq:tail_bound}
\P(K>k)\lesssim \e^{-C_K\,k\log k}\quad\text{for some $C_K>0$}.
\end{equation}
Under this prior, one can show that $\P(K>C\, K_{\nu}\C\Y^{(n)})\rightarrow 0$ in $\P_{f_0}^{(n)}$ probability (RP17). The posterior thus does not overshoot the oracle $K_{\nu}$ too much.

In the BART prior,  the distribution $\pi(K)$ is implicitly defined through the GW process rather than directly through \eqref{eq:tail_bound}. 
In order to see whether BART induces a sufficient amount of regularization, we first need to obtain a tail bound of $\pi(K)$ under the   GW process and show that it decays fast enough. 
One seemingly simple remedy would be to set $\gamma=0$ (which coincides with the homogeneous GW case) and $\alpha=c/n$ with some $c>0$. Standard branching process theory then implies
$
\Pi(K>k)\lesssim \e^{-C_K\,k\log n}.
$
This prior is more aggressive than \eqref{eq:tail_bound}. Moreover, letting the split probability $p_{split}(\Omega_k)$ decay with sample size is counterintuitive.  By choosing $\alpha=c$, on the other hand, one obtains  $\Pi(K>k)\lesssim \e^{-C_K\,k}$ which is not  aggressive enough.

While the homogeneous GW processes have been studied quite extensively,  the  literature on tail bounds  for {\sl heterogeneous}  GW processes (for when $\gamma\neq 0$) has been relatively deserted.
We first review one interesting approach in the next section and then come up with a new bound in Section \ref{sec:walks}.

\subsection{Tail Bounds \`{a} la Agresti}
\cite{agresti1975extinction} obtained bounds for the extinction time distribution of  branching processes with independent non-identically distributed environmental random variables $Y_{ti}$. 

\begin{theorem}{\citep{agresti1975extinction}} 
Consider  the {heterogeneous Galton-Watson branching process} with offspring p.g.f.'s $\{g_j(s); j \ge 0\}$ satisfying $g_j^{''}(1) < \infty$ for $j \ge 0$. Denote $P_t=\prod_{j=0}^{t-1} g_j'(1)$. 
Then
\begin{equation}\label{eq:bound}
 \P(T_{ex} > t)\leq \left[P_t^{-1}+\frac{1}{2} \sum_{j=0}^{t-1} (g_j^{''}(0)/g_j'(1)P_{j+1})\right]^{-1}.
\end{equation}
\end{theorem}

Using this result, we can obtain a tail bound on the extinction time under the Bayesian CART prior.

\begin{corollary}
For the {heterogeneous Galton-Watson branching process} with offspring p.g.f.'s \eqref{pgf} with \eqref{split_prob} we have
\begin{equation}\label{eq:tex}
 \P(T_{ex}>t) < C_0\left(\frac{t^{\gamma}}{2\alpha \e^\gamma}\right)^{-t}
 \end{equation}
 for a positive constant $C_0$ that depends on $\alpha$ and  $\gamma$.
\end{corollary}

\begin{proof}
We have $g_0(s)=s$ and for $j\ge 1$
\begin{equation*}
\begin{split}
g_j(s) & =1-\alpha(1+j)^{-\gamma}+s^2\alpha(1+j)^{-\gamma},\\
 g_j^{'}(s) &= 2s\alpha (1+j)^{-\gamma}, \\
 g_j^{''}(s) &= 2\alpha (1+j)^{-\gamma}.
\end{split}
\end{equation*}
Thus we have $g_0'(1)=1$ and $g_j^{'}(1)=g_j^{''}(0)=2\alpha (1+j)^{-\gamma}$  for $j\ge 1$. Then we can write{
\begin{equation}\label{Pn_inverse}
 P_t^{-1} =\frac{\prod_{i=0}^{t-1} (1+i)^{\gamma}}{(2\alpha)^{t}} = \frac{(t!)^{\gamma}}{(2\alpha)^{t}}
\end{equation}}
and 
\begin{equation*}
 \sum_{j=0}^{t-1} (g_j^{''}(0)/g_j'(1)P_{j+1})=\sum_{j=0}^{t-1}\frac{1}{P_{j+1}} =\sum_{j=1}^{t}\frac{(j)!^{\gamma}}{(2\alpha)^{j}}> \frac{(t!)^\gamma}{(2\alpha)^t}.
\end{equation*}
Using \eqref{Pn_inverse} and the fact that $t!>(t/\e)^t\e$, we can upper-bound the right hand side of \eqref{eq:bound} with $C_0[t^\gamma/(\e^\gamma 2\alpha)]^{-t}$. 
\end{proof}
\begin{remark}
A simpler bound on the extinction time can be obtained using Markov's inequality as follows:
$
\P(T_{ex}>t)=\P(Z_t\geq 1)\leq \E Z_t\leq (2\alpha)^t[(t+1)!]^{-\gamma}.
$
\end{remark}
Using the upper bound in \eqref{eq:Tex_X} we immediately conclude that 
$$
\P(K>k)<\P(T_{ex}>\log_2 k)\leq C_0\left(\frac{\log_2^{\gamma}k}{2\alpha \e^\gamma}\right)^{-\log_2k}.
$$
{This decay, however, is not fast enough as we would ideally like to show \eqref{eq:tail_bound}.
We try a bit different approach in the next section.}

\subsection{Trees  as Random Walks}\label{sec:walks}
There is a curious connection between branching processes and random walks (see e.g. \cite{dwass1969total}). 
Suppose that a binary tree $\mT$ is revealed in the following node-by-node  exploration process:
one exhausts all nodes in generation  $d$ before revealing nodes in generation $d+1$. 
Namely, nodes are  implicitly numbered (and explored) according to their priority and this is done in a top/down manner according to their layer and a left-to-right manner within each layer  (i.e. $\Omega_0$ is the root node and, if split,  $\Omega_1$ and $\Omega_2$  are the two children (left and right) etc.)

Nodes that are waiting to be explored can be organized in a queue $Q$.  We say that a node is  {\sl active} at time $t$ if it resides in a queue. 
Starting with one active node at $t=0$ (the root node), at each time  $t$ we deactivate (remove from $Q$) the node with the highest priority (lowest index) and add its children to $Q$. 
Letting $S_t$ be the number of active nodes at time $t$, one finds that $\{S_t\}$ satisfies
$$
S_t=S_{t-1}-1+Y_{t}, \quad t\geq 1,
$$
and $S_0=1$, where $Y_{t}$ are sampled from the offspring distribution. For the {\sl homogeneous} GW process,
$S_t$ is an actual random walk where $Y_t$ are iid with a probability generating function \eqref{pgf}. 
 For the {\sl heterogeneous} GW process, $S_t$ is not strictly a random walk in the sense that $Y_{t}'s$ are not iid. Nevertheless, using this construction one can see that the total population $X$ equals the first time the queue is empty:
$$
X=\mathrm{min}\{t\geq0: S_t=0\}.
$$
Linking Galton-Watson trees to random walk excursions in this way, one can obtain a useful tail bound of the distribution of  the population size $X$.
While perhaps not surprising, we believe that this bound is new, as we could not find any equivalent in the literature.

\begin{lemma}\label{lemma:tail}
Denote by $X$ the total population size \eqref{eq:pop_size} arising from the heterogeneous {Galton-Watson} process. Then we have for any $c>0$
\begin{equation}\label{eq:tail}
\P(X>k)  \leq\e^{-k\,c+(\e^{2c}-1)\mu},
\end{equation}
where $\mu=\sum_{i=1}^{k}p_i$ and $p_i=p_{split}(\Omega_i)$, where nodes $\Omega_i$ are ordered in a top-down left-to-right fashion.
\end{lemma}
\begin{proof}
For $k>0$, we can write
$$
\P(X>k)\leq \P(S_k>0)=\P\left(\sum_{i=1}^{k}Y_i>k-1\right),
$$
where $X$ is the number of all nodes (internal and external) in the tree and $Y_i$ has a two-point distribution characterized by $\P(Y_i=2)=1-\P(Y_i=0)=p_i$.
 Using the Chernoff bound, one deduces that for any $c>0$
{\begin{align*}
&\P\left(\sum_{i=1}^{k}Y_i>k-1\right)\leq \e^{-k\,c}\,\E\e^{c\,\sum_{i=1}^{k}Y_i}=\e^{-k\,c}\,\prod_{i=1}^{k}[p_i\e^{2c}+1-p_i]
\leq\e^{-k\,c+(\e^{2c}-1)\mu}
\end{align*}}
where $\mu=\sum_{i=1}^{k}p_i$. 
\end{proof}

The goal throughout this section has been to understand whether the Bayesian CART prior of \cite{chipman1998bayesian} yields \eqref{eq:tail_bound} for some $C_K>0$.
The prior assumes $p_i=\alpha/(1+d(\Omega_i))^\gamma$.
{Choosing $c=(\log k)/2$ in \eqref{eq:tail}, the right hand side will be smaller than $\e^{-a\, k\log k}$, for some suitable $0<a<1/2$,  as long as $\mu\leq(1/2-a)\log k$.
We note that 
$$
\mu=\sum_{i=1}^kp_i<\sum_{d=1}^{\lceil \log_2 k\rceil} \frac{\alpha}{(1+d)^\gamma}2^d.
$$} 
Because the split probability $p_i$ decreases only polynomially in depth of $\Omega_i$, this  {\sl is not enough} to ensure $\mu<(1/2-a)\log(k)$. The optimal decay,  however,  will be guaranteed if we instead choose 
{\begin{equation}\label{eq:new_psplit}
p_{split}(\Omega)\propto\alpha^{d(\Omega)}\quad\text{ for some $0<\alpha<1/2$}.
\end{equation}}
To conclude, from our considerations it is not clear that the Bayesian CART prior of \cite{chipman1998bayesian} has the optimal tail-bound decay. {The following Corollary certifies that the optimal tail behavior  can be obtained with a suitable modification of $p_{split}(\Omega)$. }
\begin{corollary}\label{corollary}
Under the Bayesian CART prior of \cite{chipman1998bayesian} with \eqref{eq:new_psplit}, we  obtain \eqref{eq:tail_bound}.
\end{corollary}
\proof Follows from the considerations bove and from \eqref{eq:progeny}.

\begin{figure}
\centering
\scalebox{0.35}{\includegraphics{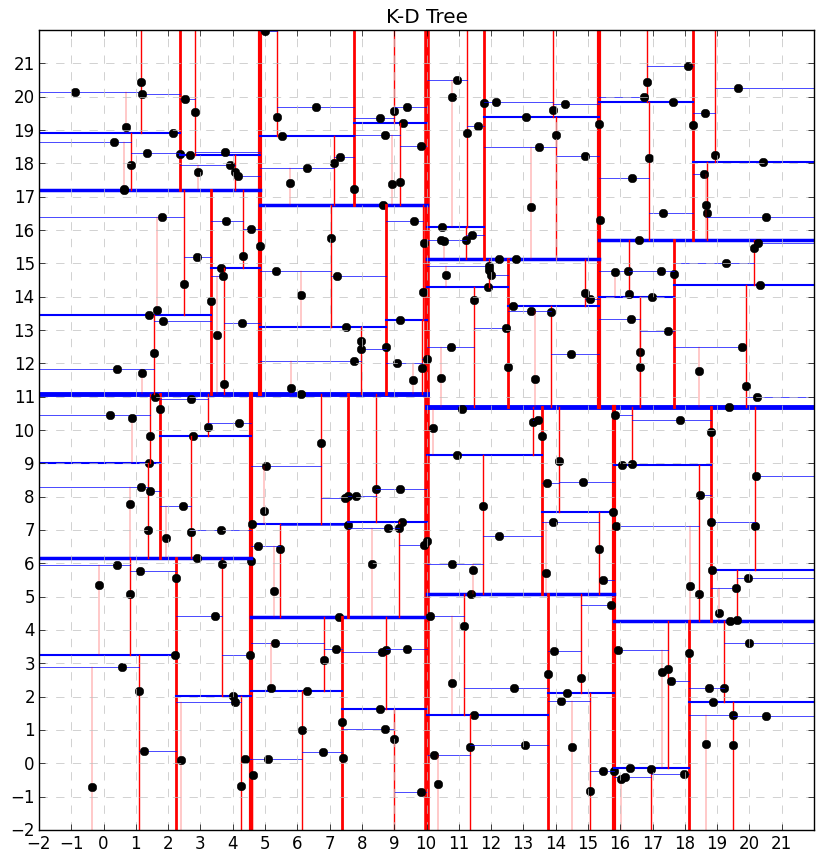}}
\caption{The $k$-$d$ trees in two dimensions at various resolution levels.}\label{fig:kd}
\end{figure}

\vspace{-0.5cm}
\section{Prior Concentration for BART}\label{sec:prior_conc}
One of the prerequisites for optimal posterior concentration \eqref{concentration} is optimal prior concentration (Condition \eqref{eq:prior1}). This condition ensures that there is enough prior support around the truth. It can be verified by constructing one approximating tree and by showing that it has enough prior mass. RP17 use the $k$-$d$ approximating tree (Remark 3.1), which is {a  balanced} full binary tree which partitions $[0,1]^p$ into nearly identical rectangles {(in sufficiently regular designs). This tree can be regarded as the most regular partition that can be obtained by splitting at observed values. A formal definition of the $k$-$d$ tree is below and a few two-dimensional examples\footnote{Source: \texttt{https://salzis.wordpress.com/2014/06/28/kd-tree-and-nearest-neighbor-nn-search-2d-case/}} (at various resolution levels)  are in Figure \ref{fig:kd}.}

\begin{definition}($k$-$d$ tree partition)
The $k$-$d$ tree partition is constructed by cycling over coordinate directions $\{1,\dots,p\}$, where all nodes at the same level are split along the same axis. For a given direction $j\in\{1,\dots,p\}$, each  internal node, say $\Omega_k$, will be split  at a median of the point set (along the $j^{th}$ axis). Each split thus roughly halves the number of points inside the cell. 
\end{definition}
After $s$ rounds of splits on each variable, all  $K$ terminal nodes have at least $\lfloor n/K \rfloor$ observations, where $K=2^{s\, p}$. The $k$-$d$ tree partitions are thus balanced in light of Definition 2.4 of \cite{rockova2017posterior} (i.e. have roughly the same number of observations inside).
The $k$-$d$ tree construction is instrumental in establishing optimal prior/posterior concentration.  Lemma 3.2 of RP17 shows that there {\sl exists} a step function supported by a $k$-$d$ partition  that safely approximates $f_0$ with an error smaller than a constant multiple of the minimax rate.  The approximating $k$-$d$ tree partition, denoted with $\wh\mT$, has $\wh K$ steps where {$\wh K\asymp n\varepsilon_n^2/\log n$ when $p\lesssim \log^{1/2}n$ (as shown in Section 8.3 of  RP17 and detailed in the proof of Theorem \ref{thm:weak_learner})}.

In order to complete the proof of posterior concentration for the Bayesian CART under {the Galton-Watson} process prior, we need to show that  $\pi(\wh\mT)\geq \e^{-c_1n\varepsilon_n^2}$ for some $c_1>0$.  This is verified in the next lemma.

\begin{lemma}\label{lemma:kdtree}
Denote with $\wh\mT$ the $k$-$d$ tree partition described above. {Assume the heterogeneous Galton-Watson process tree prior with 
$
p_{split}(\Omega_k) \propto \alpha^{d(\Omega_k)}
$ }
for some suitable $1/n\leq\alpha<1/2$. Assume $p\lesssim \log^{1/2} n$. Then we have for some suitable $c_1>0$
$$
\pi(\wh\mT)\geq \e^{-c_1\,n\,\varepsilon_n^2}.
$$
\end{lemma}
\begin{proof}
 By construction, the $k$-$d$ tree $\wh\mT$ has $\wh K=2^{p\times s}$ leaves and $p\times  s$ layers for some $s\in\mathbb{N}$ where $p$ is the number of predictors.
In addition, the $k$-$d$ tree  is complete and balanced (i.e. every layer $d$, including the last one,  has the maximal number $2^d$ of nodes). Since there are {$\wh K-1$} internal nodes and at least $1/(p\,n)$  splitting rules for each internal node,  we have
{\begin{align*}
\pi(\wh\mT)&\geq\frac{(1-\alpha^{s\,p})^{\wh K}}{(p\,n)^{\wh K-1}}\prod_{d=0}^{\log_2\wh K-1}\alpha^{2^d}
\geq\frac{(1-\alpha^{s\,p})^{\wh K}}{(p\,n)^{\wh K-1}}\alpha^{\wh K-1}\\
&\geq[\alpha(1-\alpha)]^{\wh K}\left(\frac{1}{p\,n}\right)^{\wh K-1}>\e^{-\wh K\log (2n)-(\wh K-1)\log (p\,n)}.
\end{align*}}
Since $p\lesssim \log^{1/2}n$ and {$\wh K\asymp n\,\varepsilon_n^2/\log n$} we can lower-bound the above with $\e^{-c_1\,n\varepsilon_n^2}$ for some $c_1>0$. 
\end{proof}

{For the actual BART method (similarly as in Theorem 5.1 of RP17), one needs to find an approximating {\sl tree ensemble} and show that it  has enough prior support. 
The approximating ensemble can be found in Lemma 10.1 of RP17 and  consists of $\wh\mE=\{\wh\mT_1,\dots,\wh\mT_T\}$ tree partitions obtained by chopping of branches of $\wh \mT$. {The number of trees $T$ is fixed and the trees $\mT_t$ will not overlap much when $1\leq T\leq \wh K/2$. The default BART choice $T=200$ safely satisfies this as long as $p>9$.} The little trees $\wh\mT_t$ have $\wh K^t$ leaves and  satisfy $\log_2\wh K+1\leq \wh K^t\leq \wh K$ (depending on the choice of $T$). Using Lemma \ref{lemma:kdtree} and the fact that the trees are independent a-priori (from \eqref{eq:joint_prior}) and that $T$ is fixed, we then obtain
\begin{align*}
\pi(\wh\mE)&\geq \e^{-\sum_{t=1}^T[\wh K^t\log 2n+(\wh K^t-1)\log (p\,n)]}\\
&>\e^{-T\wh K\log 2n-T(\wh K-1)\log(p\,n)}>\e^{-c_2\,n\varepsilon_n^2}
\end{align*}
for some $c_2>0$.
}
The BART prior thus concentrates enough mass around the  truth. Condition \eqref{eq:prior1} also requires verification that the prior on jump sizes concentrates around the forest sitting on $\wh\mE$. This follows directly from Section 9.2 of RP17. We detail the steps in the proof of Theorem \ref{thm:weak_learner}.
\vspace{-0.5cm}
\section{Posterior Concentration for BART}
We now have all the ingredients needed to state the posterior concentration result for BART. The result is {\sl different} from Theorem 5.1 of RP17 because here we (a) assume that $T$ is fixed, (b) assume the branching process prior on $\mT$ and (c) we do not have subset selection uncertainty. 
{ We will treat the design  as fixed  and  {\sl regular}  according to  Definition 3.3 of RP17. Moreover, the BART prior support will be restricted to $\delta$-valid ensembles with $\delta\geq1$.}

\begin{theorem}\label{thm:weak_learner}(Posterior Concentration for BART)
Assume that $f_0$ is $\nu$-H\"{o}lder continuous with  {$0<\nu\leq 1$} where $\|f_0\|_\infty\lesssim \log^{1/2} n$. Assume a {\sl regular} design $\{\x_i\}_{i=1}^n$ where  $p\lesssim \log^{1/2}n$. Assume the BART prior with $T$ fixed  and  with $p_{split}(\Omega_t)=\alpha^{d(\Omega_t)}$ for $1/n\leq \alpha<1/2$.   With $\varepsilon_n=n^{-\alpha/(2\alpha+p)}\log^{1/2} n$ we have 
 {\begin{equation*}
\Pi\left( f_{\mE,\B} \in \mathcal{F}: \|f_0 - f_{\mE,\B}\|_n > M_n\,\varepsilon_n \mid \Y^{(n)}\right) \to 0
\end{equation*}}
 for any $M_n \to \infty$  in $\P_{f_0}^{(n)}$-probability, as $n,p \to \infty$. 
\end{theorem}
\begin{proof}
Section \ref{proof:thm:weak_learner}. 
\end{proof}
Theorem \ref{thm:weak_learner} has  very important implications. It provides a frequentist theoretical justification for BART claiming that the posterior is wrapped around the truth and its learning rate is near-optimal.
As a by-product, one also obtains a statement which supports the empirical observation that BART is resilient to overfitting.

\begin{corollary}\label{corollary2}
Under the assumptions of Theorem \ref{thm:weak_learner}  we have 
$$
\Pi\left( \bigcup_{t=1}^T\{K^t>C\, n^{p/(2\nu+p)}\}  \mid \Y^{(n)}\right) \to 0
$$ 
in $\P_{f_0}^{(n)}$-probability, as $n,p \to \infty$, for a suitable constant $C>0$.
\end{corollary}
{\begin{proof}
The proof follows from the proof of Theorem \ref{thm:weak_learner} and Lemma 1 of \cite{ghosal2007convergence}. 
\end{proof}}
{In other words, the posterior distribution rewards ensembles that consist of small trees whose size does not overshoot the optimal number of steps {$K_\nu=n^{p/(2\nu+p)}$} by much.
In this way, the posterior is fully adaptive to unknown smoothness, not overfitting in the sense of  split overuse.}

\section{Discussion}
In this work, we have built on results in \cite{rockova2017posterior} to show optimal posterior convergence rate of the BART method  in the $\|\cdot\|_n$ sense. We have proposed a minor modification of the prior that guarantees this optimal performance. 
 Similar results have been obtained for other Bayesian non-parametric constructions such as Polya trees (\cite{castillo2017polya}), Gaussian processes (\cite{van2008rates}, \cite{castillo2008lower}) and deep   ReLU neural networks \citep{polson_rockova}. 
Up to now, the increasing popularity of BART has relied on its practical performance across a wide variety of problems. The goal of this and future theoretical developments is to 
 establish BART as a rigorous statistical tool with  solid theoretical guarantees. Similar guarantees have been obtained for variants of the traditional forests/trees by multiple authors including \cite{gordon_olshen1,gordon_olshen2,donoho,biau_etal,scornet,wager}. Our posterior concentration results   break the path towards establishing other theoretical properties such as Bernstein-von Mises theorems (semi and non-parametric) and/or uncertainty quantification statements. 

\section{Proof of Theorem \ref{thm:weak_learner}}\label{proof:thm:weak_learner}
The proof follows from   Lemma \ref{lemma:kdtree}, Lemma \ref{lemma:tail} and a modification proof of Theorem 5.1 of RP17. Below, we outline the backbone of the proof and highlight those places where the proof of RP17 had to be modified.
Our approach consists of establishing conditions \eqref{eq:entropy1}, \eqref{eq:prior1} and \eqref{eq:remain1} for $\varepsilon_n=n^{-\alpha/(2\alpha+p)}\log^{1/2} n$.
The first step requires constructing the sieve $\mathcal{F}_n\subset \mF$. For  a given $n\in\N$, $T\in\N$  and a suitably large integer $k_n$ (chosen later), we define the sieve
as follows:
\begin{equation}\label{sieve2}
\mF_n=\bigcup\limits_{\substack{\bm{K}:  K^t\leq k_n}}\bigcup_{\mE\in\mV\mE^{\bm K}} \mF(\mE),
\end{equation}
 where $\mF(\mE)$ consists of all functions $f_{\mE,\B}$ of the form \eqref{eq:forest_mapping} that are supported on a {\sl $\delta$-valid} ensemble $\mE$.
 All $\delta$-valid ensembles consisting of $T$ trees of sizes $\bm K=(K^1,\dots, K^T)'$ are  denoted with $\mV\mE^{\bm K}$. The sieve \eqref{sieve2} is different from the one in the proof of Theorem 5.1 of RP17. Their sieve  consisted of all ensembles whose {\sl total} number of leaves was smaller than $k_n$. 
 Here, we allow for each tree individually to have up to $k_n$ leaves.

Regarding Condition \eqref{eq:entropy1}, RP17 in Section 9.1 obtain an upper bound on the covering number for $\mF(\mE)$ as well as the cardinality of $\mV\mE^{\bm K}$ which together 
yield (for some $D>0$)
\begin{align}
&\log N\Big(\tfrac{\varepsilon}{36}, \Big\{f_{\mE,\B} \in \mathcal{F}_n: \|f_{\mE,\B} - f_0\|_n < \varepsilon\Big\}, \|.\|_n\Big)<(k_n+1)T\log (n\,p\,k_n)\notag  \\
&\quad\quad\quad\quad+D\,T\,k_n\log\left(108\, \sqrt{T\,k_n} n^{1+\delta/2}\right). \label{overall_bound}
\end{align}
With the choice $k_n=\lfloor  \wt C n\varepsilon_n^2/\log n\rfloor\asymp n^{p/(2\alpha+p)}$ (for a  large enough constant $\wt C>0$), fixed $T\in\N$ and assuming $p\lesssim \log^{1/2} n$, the
Condition \ref{eq:entropy1} will be met. 

Next, we wish to show that the prior assigns enough mass around the truth in the sense that
\begin{equation}\label{eq:prior}
\Pi(f_{\mE,\B} \in \F : \|f_{\mE,\B} - f_0\|_n \leq \varepsilon_n) \geq \e^{-d\,n\varepsilon_n^2} 
\end{equation}
for some large enough  $d>2$. 
We establish this condition by finding a lower bound on the prior probability in \eqref{eq:prior}, using only step functions supported on a single ensemble.
According to Lemma 10.1 of RP17 there exists a $1$-valid  tree ensemble $f_{\wh\mE,\wh \B}$ that approximates $f_0$ well in the sense that
$$
\|f_0 - f_{\wh{\mE},\,\wh{\mB}}\|_n \leq   ||f_0||_{\Ha}C\, p/\wh K^{\alpha/p}
$$ 
for some $C>0$,  where $\|f_0\|_{\Ha}$ is the H\"{o}lder norm and where $\wh K=2^{s\,p}$ for some $s\in\N$. 
Next, we find  the smallest  $\wh K$  such that $ ||f_0||_{\Ha}C\,  p/\wh K^{\alpha/p}<\varepsilon_n/2$.  This value will be denoted by  $a_n$ and it
 satisfies 
 \begin{equation}\label{bound}
\left(\frac{2C_0p}{\varepsilon_n}\right)^\frac{p}{\alpha} \leq a_n \leq  \left(\frac{2C_0p}{\varepsilon_n}\right)^\frac{p}{\alpha} + 1.
\end{equation}
Under the assumption $p\lesssim \log^{1/2}n$ we have $a_n\asymp n^{p/(2\alpha+p)}$.
Denote by $\wh{\mE}$ the approximating ensemble described in Section \ref{sec:prior_conc}. Next, we denote  with $\wh{\bm{K}}=(\wh{K}^1,\dots,\wh{K}^T)'$ the vector of tree sizes, where $\log_2 a_n+1\leq\wh{K}^t\leq  a_n$.
Then we can lower-bound the left-hand side of \eqref{eq:prior} with
\begin{equation}\label{lb}
\pi(\wh\mE) \Pi\left(f_{\wh{\mE},{\mB}}\in\mF(\wh{\mE}): \|f_{\wh{\mE},{\mB}} - f_0\|_n\leq \varepsilon_n\right),
\end{equation}
where $\mF(\wh{\mE})$ consists of all additive tree functions supported on $\wh{\mE}$. In Section \ref{sec:prior_conc} we show that $\pi(\wh\mE)>\e^{-c_2\,n\varepsilon_n^2}$. Moreover,
RP17 in Section 10.2 show that, for some $C>0$,
$$
\Pi\left(f_{\wh{\mE},{\mB}}\in\mF(\wh{\mE}): \|f_{\wh{\mE},{\mB}} - f_0\|_n\leq \varepsilon_n\right)>\Pi\left(\mB\in\R^{\widetilde a_n}:\|\mB-\wh{\mB}\|_2<\frac{\varepsilon_n}{2} \frac{1}{C\sqrt{\widetilde{a}_n}} \right),
$$
where  $\widetilde{a}_n=\sum_{t=1}^{{T}}\wh{K}^t\leq T\,a_n$ and where $\wh{\mB}\in\R^{\wt{a}_n}$ are the steps of the approximating additive trees from Lemma 10.1 of RP17.  This can be further lower-bounded with
\begin{equation}\label{eq:ratio2}
\e^{-\frac{\varepsilon_n^2}{8C^2\wt{a}_n}-{a}_n(C_2\|f_0\|_\infty^2 +\log2)}\left(\frac{\varepsilon_n^2}{4C^2\wt{a}_n}\right)^\frac{\widetilde{a}_n}{2}   \left(\frac{2}{\widetilde{a}_n}\right)^{\widetilde{a}_n/2 + 1}.
\end{equation}
Under the assumption $\|f_0\|_\infty\lesssim\log^{1/2}n$, this term is larger than $\e^{-D\,\wt a_n\log n}$ for some $D>0$. Since $\wt a_n\lesssim n\varepsilon_n^2$,  there exists $d>0$ such that $\Pi(f_{\mE,\B}\in\mF : \|f_{\mE,\B} - f_0\|_n \leq \varepsilon_n)>\e^{-d\,n\varepsilon_n^2}$.

Lastly,  Condition \eqref{eq:remain1} entails showing  that  $\Pi(\mF\backslash\mF_n)=o(\e^{-(d+2)\,n\varepsilon_n^2})$ for $d$ deployed in the previous paragraph. It suffices to show that 
 $$
 \Pi\left( \bigcup_{t=1}^T\{K^t>k_n\}\right)\e^{(d+2)\,n\varepsilon_n^2}\rightarrow 0.
 $$
Under the independent Galton-Watson prior on each tree partition, Corollary \ref{corollary} implies that the probability above can be upper-bounded with
$
\sum_{t=1}^T\Pi(K^t>k_n)\lesssim T\e^{-C_K\, k_n\log k_n}.
$
With $k_n\asymp n\varepsilon_n^2/\log n$ and a fixed $T\in\N$, we have $T\e^{-C_K\, k_n\log k_n+ (d+2)\,n\varepsilon_n^2}\rightarrow 0$ for $C_K$ large enough.

\bibliographystyle{plainnat}

\end{document}